%% file: manuscript.tex
\documentclass[sigconf,natbib]{acmart}

\input{customised-commands}

\title{Get Global Guarantees: On the Probabilistic Nature of Perturbation Robustness}

\author{Wenchuan Mu}
\affiliation{
  \institution{Singapore University of Technology and Design}
  \country{Singapore}
}
\author{Kwan Hui Lim}
\affiliation{
  \institution{Singapore University of Technology and Design}
  \country{Singapore}
}

\copyrightyear{2025}
\acmYear{2025}
\setcopyright{acmlicensed}\acmConference[CIKM '25]{Proceedings of the 34th
ACM International Conference on Information and Knowledge
Management}{November 10--14, 2025}{Seoul, Republic of Korea}
\acmBooktitle{Proceedings of the 34th ACM International Conference on
Information and Knowledge Management (CIKM '25), November 10--14, 2025,
Seoul, Republic of Korea}
\acmDOI{10.1145/3746252.3761039}
\acmISBN{979-8-4007-2040-6/2025/11}

\begin{document}

\begin{abstract}

In safety-critical deep learning applications, robustness measures the ability of neural models that handle imperceptible perturbations in input data, which may lead to potential safety hazards. Existing pre-deployment robustness assessment methods typically suffer from significant trade-offs between computational cost and measurement precision, limiting their practical utility. To address these limitations, this paper conducts a comprehensive comparative analysis of existing robustness definitions and associated assessment methodologies. We propose tower robustness to evaluate robustness, which is a novel, practical metric based on hypothesis testing to quantitatively evaluate probabilistic robustness, enabling more rigorous and efficient pre-deployment assessments. Our extensive comparative evaluation illustrates the advantages and applicability of our proposed approach, thereby advancing the systematic understanding and enhancement of model robustness in safety-critical deep learning applications.

\end{abstract}

\maketitle

\section{Introduction}
\label{sec:intro}

Deep learning has transformed numerous fields, enabling breakthroughs in image recognition~\cite{he2016deep,dosovitskiy2020image}, natural language processing~\cite{zhang2020pegasus,floridi2020gpt}, and autonomous systems~\cite{fu2016using}. In safety-critical domains, such as self-driving vehicles~\cite{kurakin2018adversarial}, medical diagnosis~\cite{ragoza2017protein,wen2020adapting}, and industrial control systems~\cite{al2020ensemble}, the reliability and robustness of deep neural networks (DNNs) are paramount. In these contexts, even minor perturbations in input data can lead to severe consequences, highlighting the need for rigorous robustness evaluation prior to deployment.

Model robustness evaluation is well-studied but remains challenging. Formal verification is one established approach, but it often suffers from incomplete problem formulations~\cite{katz2017reluplex}, failing to capture the full range of real-world perturbations. In addition, verification techniques are typically computationally expensive, rendering them impractical for large-scale models~\cite{zhang2023proa}. These limitations call for evaluation methods that are both more comprehensive and computationally feasible.

To address these gaps, we explore probabilistic robustness assessment, which aims to estimate a model's failure probability under adversarial perturbations. While existing probabilistic methods frequently rely on approximations, such approximations risk overlooking critical adversarial instances~\cite{agresti1998approximate}, potentially overestimating model robustness.

We propose a new framework that employs Tower robustness with hypothesis testing to yield statistical guarantees on robustness estimates. This approach enables precise quantification of the probability of model failure, offering a more informative and actionable robustness measure. By grounding our evaluation in exact statistical methods, we aim to better capture model behaviour in real-world deployments.

We make our method available at \url{https://github.com/confidential-submission/adversarial-attacks-pytorch}, offering researchers and practitioners a user-friendly and statistically rigorous tool for robustness assessment.

We validate our approach through extensive experiments on several large-scale DNNs, comparing it against state-of-the-art baselines. Results demonstrate that our method consistently delivers more accurate and reliable robustness estimates, underscoring its value for both theoretical analysis and practical deployment in high-stakes applications.

The rest of the paper is organized as follows: Section 2 reviews related work on robustness evaluation; Section 3 outlines the theoretical basis of our metric; Section 4 presents experiments and results, and; Section 5 concludes with key findings.

\begin{figure}
\begin{subfigure}{0.32\linewidth}
\centering
\includegraphics[width=\linewidth]{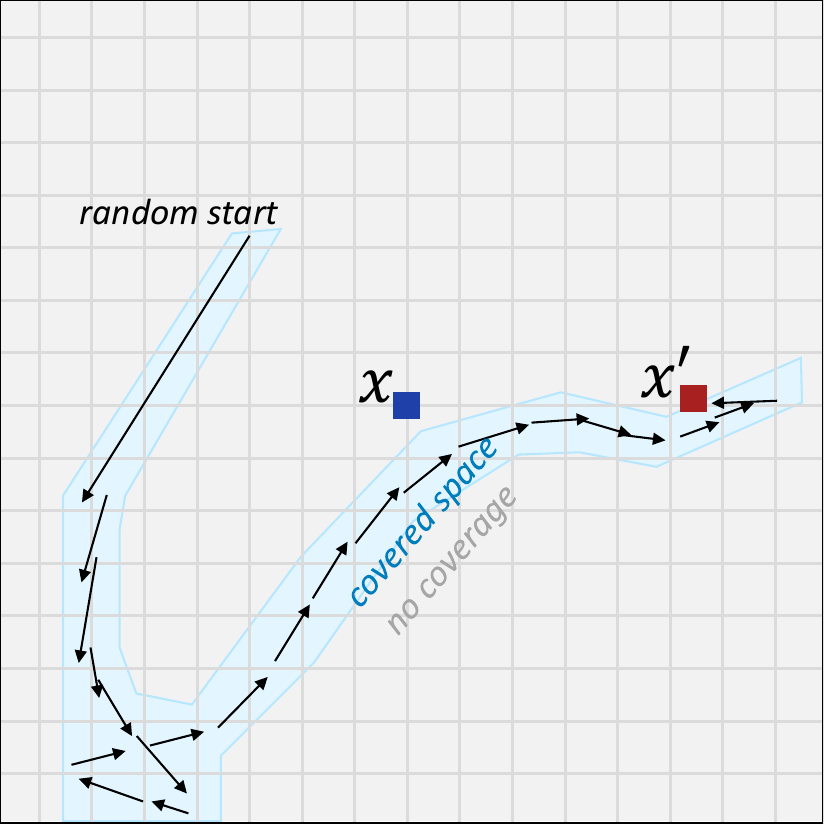}
\caption{~}
\end{subfigure}
\begin{subfigure}{0.32\linewidth}
\centering
\includegraphics[width=\linewidth]{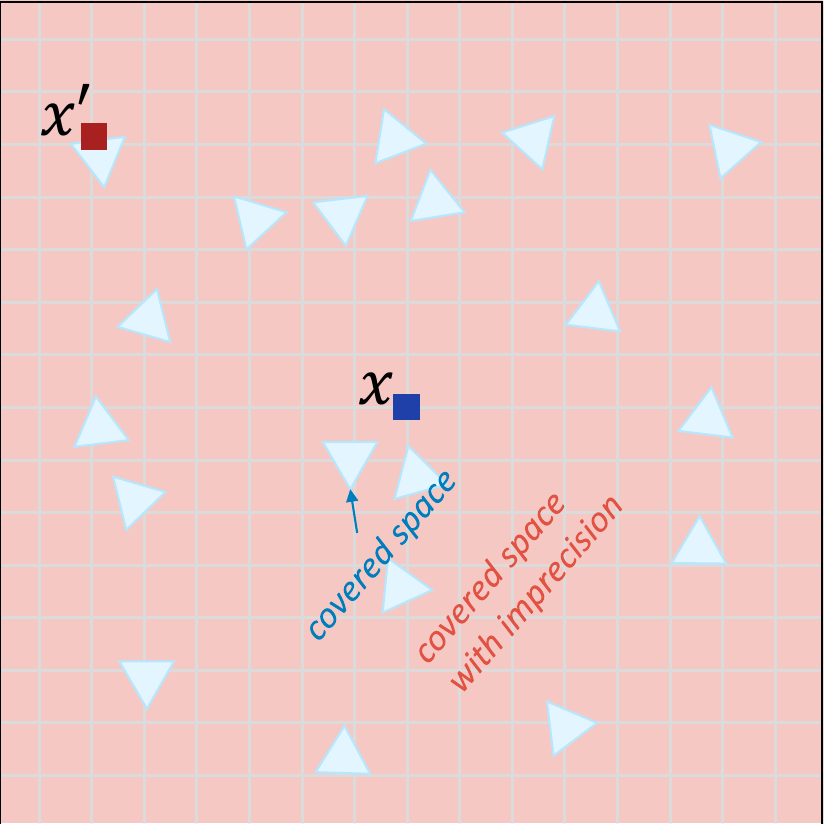}
\caption{~}
\end{subfigure}
\begin{subfigure}{0.32\linewidth}
\centering
\includegraphics[width=\linewidth]{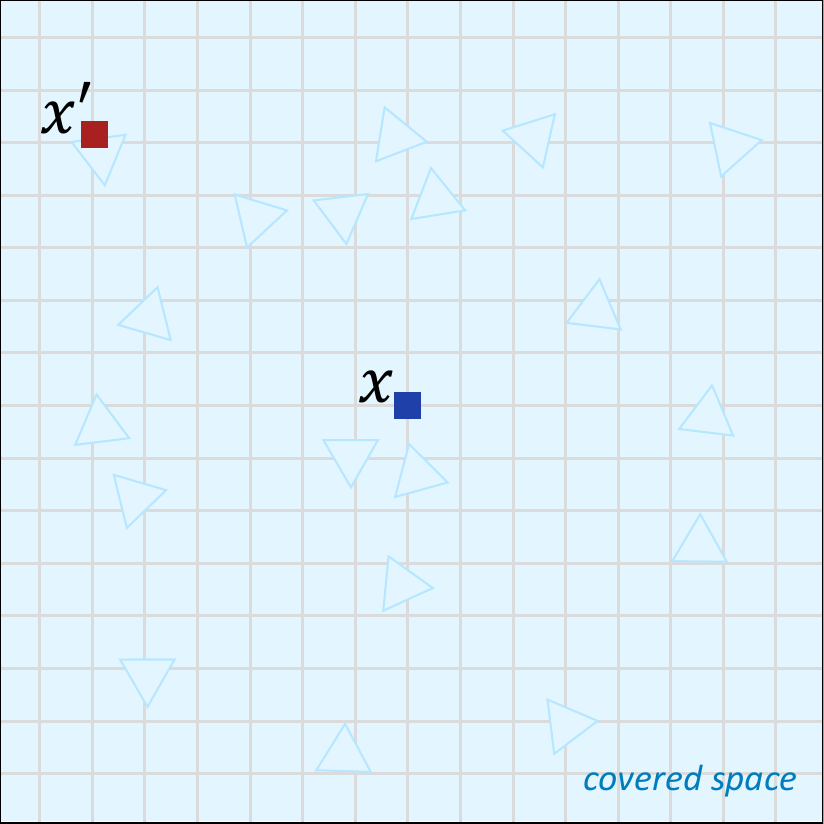}
\caption{~}
\end{subfigure}
  
  \caption{Comparison of space coverage: (a) Adversarial attacks explore a small portion in the vicinity around input sample $x$ to find adversarial example $x'$. There is no guaranteed prediction for the untouched space. (b) Existing probabilistic methods predict over the entire vicinity but lack precision. The proportion of adversarial examples is possibly higher than the predicted "guarantee". (c) The proposed method achieves comprehensive coverage and maintains precision for robustness evaluation.}
  
  \label{fig:teaser}
\end{figure}

\section{Preliminaries}

This section discusses the prerequisite knowledge. We first introduce the notations of each quantity. After that, we review the essentials of statistical learning, e.g., data sampling, objective functions, and neural networks.

\subsection{Notations}

\begin{table}[t]
    \caption{Mathematical object letter styles (e.g., using letter A) }
    \centering
    
    \begin{tabular}{l|l}
    \toprule
        A  variable or variable function               & $a$          \\
        \rowcolor{mygray}A  random variable or random variable function & $\mathbf{a}$ \\
        A set                                          & $\mathbb{A}$ \\
        \rowcolor{mygray}A distribution                                 & $A$          \\
        A function whose value is $1$ if $a$ is true, $0$ otherwise     & $\mathbf{1}_{a}$          \\
    \bottomrule
    \end{tabular}
    \label{tab:numbers}
\end{table}

We use the simple letter styles as given in \cref{tab:numbers} to distinguish different types of mathematical objects. Note that the letter $P$ is reserved for probability, and $\operatorname{E}$ for expectation. The notation (e.g., for variables) is generally insensitive to whether an object is a scalar or a vector. Subscripts or superscripts in variable names are simply part of the name, not as indices. In fact, we will never index into vector variables or sets in this manuscript. Also, lowercase Greek letters are always used for scalars.

\subsection{Data, Distribution, and Samples}

In the following, we briefly review the probability and statistics background of this work. Let $D$ denote a joint probability distribution for random variables $\mathbf{x}$ and $\mathbf{y}$, such that for some positive integer $n$, a sample from $D$ could look like
\begin{equation}
    \left\{ (x_1, y_1), (x_2, y_2), \ldots, (x_n, y_n)\right\}
\end{equation}
where $n$ is the sample size. In practice, we are usually given the sample. For example, MNIST~\cite{lecunmnist} dataset could be seen as a sample drawn from the distribution of handwritten digits\footnote{To view MNIST as an unbiased sample, we might as well say it is from a distribution of American employees' and high school students' handwritten digits.}. If we further look into the set of MNIST training images and testing images, they can be considered two samples drawn from this distribution where $n=60,000$ for the training set and $n=10,000$ for the testing set.

Let $\mathbb{X}$ be the outcome space of $\mathbf{x}$, and $\mathbb{Y}$ be the outcome space of $\mathbf{y}$. Intuitively, in the MNIST context, $\mathbb{X}$ would be the set containing all 28$\times$28 pixel greyscale images, and $\mathbb{Y}$ would be a set of integers from 0 to 9.

Often, we are interested in the expected value of some function $f(\mathbf{x}, \mathbf{y})$, denoted by $\operatorname{E}(f(\mathbf{x}, \mathbf{y}))$ and computed as follows.
\begin{equation}
    \operatorname{E}\left[f(\mathbf{x}, \mathbf{y})\right] = \int_{\mathbb{Y}} \int_{\mathbb{X}} f(x, y) P(\mathbf{x}=x, \mathbf{y}=y) ~dx ~dy
\end{equation}
In practice, if a sample is large enough, the law of large numbers ensures that the sample average converges to the true expected value, i.e., as $n\to\infty$
\begin{equation}
\label{eq:expected}
    \operatorname{E}\left[f(\mathbf{x}, \mathbf{y})\right] \approx \frac{1}{n}\sum_{i=1}^n  f(x_i, y_i).
\end{equation}

\subsection{Learning and Neural Models}

Let us denote a function $h:\mathbb{X}\to\mathbb{Y}$. Essentially, $\mathbb{X}$ is considered an input feature space, and $\mathbb{Y}$ is considered a label space, while a model function $h$ predicts a label from a given feature. A good model makes correct predictions, i.e., we want to maximise
\begin{equation}
\label{eq:correct}
    \operatorname{E}\left[1_{h(\mathbf{x}) = \mathbf{y}} \right].
\end{equation}
Its practical measurement, however, often relies on a finite set. According to approximation~(\ref{eq:expected}), we can rewrite expression~(\ref{eq:correct}) as expression~(\ref{eq:correct-sample}). Indeed, expression~(\ref{eq:correct-sample}) denotes the commonly used test accuracy metric when the set is the testing set.
\begin{equation}
\label{eq:correct-sample}
    \frac{1}{n}\sum_{i=1}^n  1_{h(x_i) = y_i}
\end{equation}

The process of actively optimising the model is referred to as learning (or training). Learning is complicated and may be done empirically through various forms~\cite{vapnik1999nature,radford2021learning,chen2020simple,deng2019arcface}. This is primarily because we only have access to a finite training set of data that theoretically does not suffice to calculate the objective function, i.e. expression~(\ref{eq:correct}) or (\ref{eq:correct-sample}). Consequently, a common practice is to define an optimisation problem (typically the objective function) that is fully measurable from the training set. For instance, the empirical risk minimisation (ERM~\cite{vapnik1999nature}) method is defined as follows, which is one of the most straightforward learning strategies.

\begin{definition}[Empirical risk minimisation]
    Define a density function $g:\mathbb{X}\times\mathbb{Y}\to[0,1]$ such that $\forall x\in\mathbb{X},~ 1 = \sum_{y\in\mathbb{Y}} g(x,y)$. Supposedly, this density reflects how likely a prediction on $x$ will be $y$, i.e., $g(x, y) \equiv P(\hat{\mathbf{y}} = y\mid \mathbf{x} = x)$, where $\hat{\mathbf{y}}$ denotes the random variable of prediction. In this way, expression~(\ref{eq:correct}) can be rewritten as $\operatorname{E}_{\mathbf{x}}\left[ \left\langle P(\mathbf{y}\mid \mathbf{x}), P(\hat{\mathbf{y}}\mid \mathbf{x})\right\rangle\right]$, where $\langle\,,\rangle$ stands for inner product. Then, instead of directly maximising this inner product, ERM minimises the Kullback-Leibler divergence (KL divergence) between these two conditional probabilities. When it comes to the training sample with size $n$, the term to be minimised is as follows.
    \begin{equation}
        -\frac{1}{n}\sum_i^n \log g(x_i, y_i)
    \end{equation}
    Obtaining $h$ from $g$ can be straightforward, e.g., one common practice is to take $h(x) \coloneqq \arg\max_y g(x, y)$.
\end{definition}

Another key aspect of learning, which is orthogonal to the choice of the objective function, is the model architecture. While the optimisation process (e.g., minimising empirical risk) is general, the way we represent the model function $h$ can vary significantly, and some representations may perform better than others in practice~\cite{he2016deep}. A common and effective approach is to use a parametrised model, such as a neural network. In a simplified view, a neural network consists of (1) a fixed structural design that defines the topology of the network (e.g., number of layers, connections), and (2) variable parameters such as numerical weights and biases. During training, the optimisation algorithm searches over the parameter space, while the structure remains unchanged.

\begin{definition}[Neural Network Models]
    If $h$ is a neural model, we can re-express it as $h(x) \coloneqq \arg\max_y g(x, w, y)$, where $g:\mathbb{X}\times\mathbb{W}\times\mathbb{Y}\to[0,1]$ and $\forall x\in\mathbb{X}, \forall w\in \mathbb{W}, ~ 1 = \sum_{y\in\mathbb{Y}} g(x,w,y)$. Set $\mathbb{W}$ is the entire search space of parameters. Generally, the (sub)gradients $\partial g/\partial w$ and $\partial g/\partial x$ are both available.
\end{definition}

\subsection{Adversarial Examples}

Maximising the probability of making correct predictions is generally a sound and intuitively reasonable objective. Yet, recent research has raised important questions about whether this objective is sufficient on its own. One concern arises from the context of adversarial examples. Inputs that have been slightly perturbed in ways often imperceptible to humans, yet cause a model's prediction to flip from correct to incorrect. This could be a security problem in practice, e.g., a subtly altered stop sign may be misread by an autonomous vehicle as a speed limit sign. 

Although it remains an open question whether the adversarial example (AE) issue can be solved inherently by optimising expression~(\ref{eq:correct}), empirical evidence shows a troubling trade-off. Some models with high test accuracy perform poorly under adversarial perturbations, while others with better adversarial example performance often suffer from lower test accuracy. Here, we describe the formal definitions of adversarial examples for subsequent analysis.

\begin{definition}[Adversarial Example]
\label{def:ae}
    For real number $p\ge 1$, $\epsilon \in [0,1]$, given any input $x$ from space $\mathbb{X}$, we have that $x'$ a (measurable) $(p,\epsilon)$-neighbour of $x$ if and only if $x'\in\mathbb{X} \land \norm{x-x'}_p \le \epsilon $. For any model $h:\mathbb{X}\to\mathbb{Y}$, for any \emph{paired} $(x, y)\in\mathbb{X}\times\mathbb{Y}$, we have that a $(p,\epsilon)$-neighbour $x'$ of $x$ is also a $(p,\epsilon)$-AE of $x$ (regarding this $h$ and this $y$) if and only if $h(x) = y ~\land  h(x')\neq y$. Additionally, given $(x, y)$, let us define a condition $c_{\text{AE}}(x, y; p, \epsilon)$ or simply $c_{\text{AE}}$ as there exists some $(p,\epsilon)$-AE of $(x,y)$. Formally,
    \begin{equation}
    \begin{aligned}
        c_{\text{AE}}(x, y; p, \epsilon, h) &\coloneqq (h(x) = y) ~\land \\
        &\exists x'.~ x'\in\mathbb{X} \land \norm{x-x'}_p \le \epsilon \land   h(x')\neq y
    \end{aligned}
    \end{equation}    
\end{definition}

\cref{def:ae} is often the most intuitive and widely accepted way to define adversarial examples, as it aligns well with our understanding of small, imperceptible perturbations causing misclassification. However, it introduces a subtle issue, i.e., what if the original input is misclassified, yet there exists a neighbour that is correctly classified? Does that imply we should treat the original input as the ``adversarial'' point from the perspective of the neighbour?

Despite this ambiguity, \cref{def:ae} is not flawed. We adopt this definition as a foundation and demonstrate different contexts of robustness are given in \cref{tab:definitions}.

\begin{table*}
    
    \caption{Various contexts of robustness. In any case, we are given a pair of input and label $(x, y)$ from $\mathbb{X}\times\mathbb{Y}$, and $h$ is the model.}
    \label{tab:definitions}
\resizebox{\textwidth}{!}{
    
    \begin{tabular}{>{\centering\arraybackslash}p{30mm}|>{\centering\arraybackslash}p{70mm}|>{\centering\arraybackslash}p{20mm}|>{\centering\arraybackslash}p{125mm}}
    \toprule
    Quantity & Event formula & Computable & Meaning\\
    \midrule
    Adversarial robustness
    & $c_\text{AR}(x, y; p, \epsilon, h) \coloneqq (h(x) = y) ~\land \lnot\exists x'.~ x'\in\mathbb{X} \land \norm{x-x'}_p \le \epsilon \land h(x')\neq y$
    & Intractable
    & $c_\text{AR}$ is true if there does not exist any adversarial example to this input-label pair, and false if there do.\\
    \midrule
    Adversarial attack failure rate
    & $c_\text{AF}(x, y; p, \epsilon, h, f_\text{attack}) \coloneqq \lnot\exists x'.~ x'\in f_\text{attack}(x, y, p, \epsilon, h)  \land \norm{x-x'}_p \le \epsilon \land h(x')\neq y$, where $f_\text{attack}(x, y, p, \epsilon, h)\subset\mathbb{X}$
    & Yes
    & $c_\text{AF}$ being false means we successfully find at least one adversarial example. $c_\text{AF}$ being true means no adversarial example has been found yet, but we have no idea whether one will be found.\\
    \midrule
    Adversarial accuracy
    & $c_\text{AA}(x, y; p, \epsilon, h, f_\text{attack}) \coloneqq (h(x) = y) ~\land \lnot\exists x'.~ x'\in f_\text{attack}(x, y, p, \epsilon, h) \land \norm{x-x'}_p \le \epsilon \land h(x')\neq y$
    & Yes
    & $c_\text{AA}$ indicates the prediction on $x$ is correct, and no adversarial example has been found yet (again, we do not know if one will be found). Otherwise, we get $\vdash \lnot c_\text{AA}$.\\
    \midrule
    Deterministic certified robust accuracy
    & To find $c_\text{DC}(x, y; p, \epsilon, h)$ such that $\vdash c_\text{DC}(x, y; p, \epsilon, h) \to c_\text{AR}(x, y; p, \epsilon, h) $.
    & Yes, but not for all $x$
    & Try to let a sufficient condition of $c_\text{AR}$ hold. Then, if $c_\text{DC}$ holds, we know that no adversarial exists around the input. But if $c_\text{DC}$ does not hold, we do not get any implication.\\
    \midrule
    Probabilistic robustness
    & $c_\text{PR}(x, y; p, \epsilon, h, \kappa) \coloneqq  P_{\text{Uniform}, \norm{\mathbf{x'} - x}_p\le \epsilon } ( h(\mathbf{x'})\neq y) \le \kappa$, for $0<\kappa<1/2$
    & Unverifiable from samples
    & $c_\text{PR}$ hold if and only if the proportion of adversarial examples among neighbours is capped by a real number $\kappa$. Predicting the original $x$ correctly is not emphasised.\\
    \midrule
    Probabilistic robust accuracy
    & $c_\text{PRA}(x, y; p, \epsilon, h, \kappa, \alpha) \coloneqq  P(P_{\text{Uniform}, \norm{\mathbf{x'} - x}_p\le \epsilon } ( h(\mathbf{x'})\neq y) > \kappa) \le \alpha$, for $0<\kappa<1/2, 0< \alpha < 1$.
    & Yes
    & When $c_\text{PRA}$ holds, the probability of $c_\text{PR}\neq 1$ is small, such that it is capped at a real value $\alpha$. When $c_\text{PRA}$ does not hold, it just means that we are not informed by this inequality.  \\

    \bottomrule    
  \end{tabular}
}
\end{table*}

\section{Tower Robustness}

In this section, we present Tower robustness, which is a sound probabilistic quantity of the perturbation robustness concept. First, we describe the essence of robustness and revisit various observation strategies for robustness to see why the concept of robustness has not been captured well enough. After that, we propose Tower robustness based on the Tower Law and formally show why this is a reasonable measure of robustness.

\subsection{Robustness, What Are We Actually Observing?}

While avoiding adversarial examples is commonly considered the core of robustness\footnote{There are other loosely related concepts in the name of robustness, e.g., generalisation robustness or compression robustness, which is covered in the appendix}, the quantitative definition of robustness remains an open problem. Empirically, there exist a few metrics to measure robustness as demonstrated in \cref{tab:definitions}, e.g., adversarial attack failure rate, but there has never been an agreement on which one truly represents robustness.

In the following, we aim to set this issue in a theoretical framework. We start from this often overlooked question, i.e., what exactly do we aim to maximise when we ask for greater robustness?

Note that in this study, we do not focus on specific optimisation strategies, e.g., adversarial training. Instead, we focus on the objective functions on the evaluation side, e.g., test accuracy for classification. In other words, we would like to measure the robustness of models and would not concern ourselves with how these models have been built. Let us now walk through how each robustness metric is currently defined de facto in practice.

\subsubsection*{Adversarial robustness and its bounds}

The overall adversarial robustness of a model can be expressed as $\mathbb{E}[1_{c_\text{AR}}]$, or equivalently, $P(c_\text{AR})$. Here, $P(c_\text{AR})$ represents the probability that no adversarial examples exist in the neighbourhood of a randomly chosen input. Computing this probability directly is intractable, because even simple high-dimensional data like an MNIST~\cite{lecunmnist} image can have over $10^{200}$ neighbours.

A practical approach is to adversarially attack each testing case by heuristically searching for adversarial examples within a limited number of steps, rather than exhaustively checking all neighbours. Under this method, if no adversarial example is found within the search steps (i.e., $\lnot c_\text{AA}$), we can infer $\lnot c_\text{AR}$. However, if the attack times out (i.e., $c_\text{AA}$ holds), we gain no information about $c_\text{AR}$. In other words, $P(c_\text{AA})$ serves as an upper bound for $P(c_\text{AR})$. The issue is that $P(c_\text{AA})$ depends on the specific attack method used, which may not be fully independent of the model itself, especially in cases like adversarial training where the attack procedure influences the model. As a result, ranking models based on $P(c_\text{AA})$ might not reflect their ranking in terms of the adversarial robustness, $P(c_\text{AR})$.

An alternative strategy is to define a condition $c_\text{DC}$ such that $c_\text{DC} \to c_\text{AR}$, providing a lower bound for $P(c_\text{AR})$. Unlike $c_\text{AA}$, the condition $c_\text{DC}$ does not rely on any attack process, but it is known to be heavily influenced by the model choices. When $c_\text{DC}$ does not hold, it remains possible that $c_\text{AR}$ still holds, but we have no way to confirm this.

It might be imaginably possible that we can simultaneously lift both the lower and upper bounds on $P(c_\text{AR})$ to effectively improve robustness. In practice, however, a substantial gap often remains between these bounds, and it is rare to find models that make meaningful progress on both fronts at once. As a result, these quantities offer limited insight into true robustness.

\subsubsection*{Probabilistic robustness}

Probabilistic robustness was originally proposed to address the problem that high-accuracy models usually lack robustness guarantees~\cite{robey2022probabilistically,zhang2023proa}. To this end, the focused event is modified from $c_\text{AR}$ to $c_\text{PR}$, i.e, from whether there are no adversarial examples to whether there are few enough adversarial examples. However, this also makes the definition less interpretable, e.g., it would read like maximising "the probability of misclassification upon input perturbation is less than a tolerance level"~\cite{robey2022probabilistically}. Furthermore, the choice of this tolerance level ($\kappa$ as in \cref{tab:definitions}) introduces a new parameter that not only requires justification but aggravates the complicatedness of robustness evaluation.

Worse, this $P(c_\text{PR})$ form itself is not verifiable through a finite number of samples, and the practical solution to provide such a probabilistic robustness guarantee must read like maximising "the probability of that the probability that the probability of misclassification upon input perturbation is less than a tolerance level is greater than or equal to a significance level"~\cite{zhang2023proa}. While this is a reasonable choice out of practicality, it introduces yet another new parameter ($\alpha$ as in \cref{tab:definitions}), which makes the evaluation even more confusing. These quantities risk becoming little more than shadows on the wall, approximations that hint at the underlying reality but fail to fully capture it. We believe that seeking a suitable and clean definition of robustness is worthwhile.

\subsubsection*{Robustness might not be just about rates}
As demonstrated in \cref{tab:definitions}, nearly every event involves a complex combination of multiple factors. Although it is straightforward to compute an average across all test cases, this does not necessarily mean that the events are well-defined in a probabilistic sense.

A potential blind spot is that robustness might not have to be represented as a rate (like all in \cref{tab:definitions}). Intuitively, the measurable quantity of robustness could be understood as an estimator of a deeper property. Let us take a look at the accuracy counterpart, i.e., accuracy is not just a rate, but an unbiased estimator of a probability, with clear statistical properties like expectation, variance, and concentration bounds (e.g., Hoeffding’s inequality). In the next step, we will find this property for robustness and present our definition.

\subsection{The Probability of Misprediction Given Any Neighbourhood}

In the following, we introduce Tower robustness, a property that underpins measurable notions of robustness in a meaningful and structured way. First, we state its formal definition. Then, using the Tower Law, we show how this property encapsulates robustness and connects to measurable quantities. Finally, we offer an interpretation from a convolutional perspective to aid in understanding its implications.

\begin{definition}[Tower Robustness]
\label{def:tr}
    For Tower robustness, we ask, given a random neighbourhood, what is the probability of making correct predictions? Formally, we are given a joint distribution of $\mathbf{x}, \mathbf{y}$, where $\mathbf{x}$ has support $\mathbb{X}$ and $\mathbf{y}$ has support $\mathbb{Y}$. For any model $h:\mathbb{X}\to\mathbb{Y}$, its $(p,\epsilon)$-Tower robustness is defined as 
    \begin{equation}
    \label{eq:tr-prob}
        P_{\norm{\mathbf{x'}-\mathbf{x}}_p \le \epsilon} ~\left(h(\mathbf{x'})=\mathbf{y}\right),
    \end{equation}
    or equivalently
    \begin{equation}
    \label{eq:tr-exp}
        \operatorname{E}_{\norm{\mathbf{x'}-\mathbf{x}}_p \le \epsilon}~\left[1_{h(\mathbf{x'}) = \mathbf{y}} \right].
    \end{equation}
\end{definition}

Intuitively, the Tower robustness (TR) of a model under a given distribution and specified $(p,\epsilon)$ would be a unique scalar. A higher TR suggests a more robust model at the given setting. Note that TR itself, as defined in \cref{def:tr}, is computable using finite samples, but may be estimated with a variance. Thus, we propose computable bounds of TR in \cref{sec:bounds}.

\subsubsection{Derivation of Tower robustness and link to other robustness quantities}

To understand the rationale of Tower robustness, we start from the sense of deterministic certified robust accuracy (DC accuracy). According to \cref{tab:definitions}, DC accuracy can be written as $\sum_i^n 1_{c_{DC, i}}/n$, representing the proportion of test cases for which a deterministic certification is achieved. While DC accuracy is typically viewed as unrelated to probabilistic behaviour unless explicitly noted, we will demonstrate that it is, in fact, fundamentally probabilistic in nature.

\paragraph{The probabilistic nature in deterministic robustness}

The term probabilistic robustness, when used in a deterministic context like a recent systematisation of knowledge~\cite{li2023sok}, is often ambiguous. This ambiguity arises from that there are two ways of choosing random variables. The first way uses $\mathbf{x}, \mathbf{y}$ as random variable. That is, for each given pair of data, the condition $c_{\text{DC}}$ is a (deterministic) function value, and we aim to get 
\begin{equation}
\label{eq:probdeter}
    P\left(c_{\text{DC}}(\mathbf{x}, \mathbf{y}, p, \epsilon, h)\right).
\end{equation}
This is often overlooked, because we often resort to its unbiased maximum likelihood estimation (MLE), i.e., DC accuracy. The second way uses $\mathbf{c}_{\text{DC}}$ as random variable. This particularly works when given $x, y$, we are not able to determine $c_{\text{DC}}$ and have to guess whether it is true, e.g., when robustness verification is not complete but already acquires some knowledge $\mathbf{c}_{\text{known}}$ about this testing case. This works within the neighbourhood of individual given $x, y$, rather than across samples. Formally, this could be expressed as $P(\mathbf{c}_{\text{DC}}\mid \mathbf{c}_{known}(x, y, p, \epsilon, h))$ .

In this study, we leverage the first way of choosing random variables, i.e., expression~(\ref{eq:probdeter}), to help understanding Tower robustness. Essentially, expression~(\ref{eq:probdeter}) represents the probability of achieving deterministic certification for a random input-label pair from the distribution. Furthermore, expression~(\ref{eq:probdeter}) is a sound lower bound of $P\left(c_{\text{AR}}(\mathbf{x}, \mathbf{y}, p, \epsilon, h)\right)$, i.e., adversarial robustness.

For a long time, $P\left(c_{\text{AR}}(\mathbf{x}, \mathbf{y}, p, \epsilon, h)\right)$ has been seen well capturing robustness. However, this measure comes with notable limitations. First, its value is theoretically proven to be bounded within a low range, primarily because excluding all adversarial examples is an inherently strict requirement. As a result, relying on $P\left(c_{\text{AR}}(\mathbf{x}, \mathbf{y}, p, \epsilon, h)\right)$ to quantify robustness can be misleading. Intuitively, it is akin to signing up for an extremely difficult chess match and then scoring poorly. Merely attempting a hard challenge does not demonstrate mastery. More importantly, when encountering a new random input-label pair, one cannot determine in advance whether adversarial examples are present. This uncertainty means that robustness, in this context, remains probabilistic. Worse yet, due to the typically low probability value, such a probabilistic estimate provides little practical confidence. To intuitively understand this, let us look at the following game example.

\begin{summary}[title=Example: Expected Payout from a Game,]
\label{ex:game}
\vspace{-5mm}

Imagine you play a carnival game where you are randomly assigned to either Machine A (40\% chance) or Machine B (60\% chance). Machine A pays \$10 with a 50\% chance and \$0 otherwise, while Machine B pays \$5 with an 80\% chance and \$0 otherwise. To find the expected payout, we calculate the expected value for each machine: for Machine A, it is $0.5\times10+0.5\times0=5$, and for Machine B it is $0.8\times5+0.2\times0=4$. Then, using the Tower Law, the overall expected payout is $0.4\times5+0.6\times4=2+2.4=$\$4.40.
\end{summary}

In an adversarial robustness scenario, consider a system like ``Machine A'' that, with a 50\% chance, shows zero adversarial examples, and with the remaining 50\%, shows an unknown proportion of adversarial examples, potentially many, say less than T\%. The expected adversarial example payout is then $0.5\times0+0.5\times\text{T}=$ 5T\textperthousand. Like in the carnival game, it is the 5T\textperthousand\, that matters practically. With this idea, let us review the Tower Law formally.

\begin{definition}[Tower Law]
    Let $\mathbf{b}$ be an integrable random variable, i.e., $\operatorname{E}[|\mathbf{b}|] < \infty$. As $\operatorname{E}(\mathbf{b}\mid \mathbf{a})$ is a random variable and a function of $\mathbf{a}$, we can take its expected value, and it can be shown that
    \begin{equation}
        \operatorname{E}[\operatorname{E}[\mathbf{b} \mid \mathbf{a}]] = \operatorname{E}[\mathbf{b}]
    \end{equation}
\end{definition}

\begin{theorem}
\label{thm:tower}
    We are given a joint distribution of $\mathbf{x}, \mathbf{y}$, where $\mathbf{x}$ has support $\mathbb{X}$ and $\mathbf{y}$ has support $\mathbb{Y}$. For any model $h:\mathbb{X}\to\mathbb{Y}$, its $(p,\epsilon)$-Tower robustness can be equivalently written as
    \begin{equation}
        \operatorname{E}\left[P\left(h(\mathbf{x'}) = \mathbf{y} ~\left|~ \norm{\mathbf{x'}-\mathbf{x}}_p \le \epsilon\right.\right)\right]
    \end{equation}
\end{theorem}

\begin{proof}
    We apply the Tower Law to expression~(\ref{eq:tr-exp}). First, we can see that the premise holds, i.e., $1_{(*)}$ returns either 0 or 1, which makes it an integrable function. Then, we confirm that for random variables $\mathbf{x'}$ and $\mathbf{y}$, $\norm{\mathbf{x'}-\mathbf{x}}_p \le \epsilon$ and $1_{h(\mathbf{x'}) = \mathbf{y}}$ are indeed a function of $(\mathbf{x}, \mathbf{y})$. Then, we get
    \begin{equation}
        \operatorname{E}\left[\operatorname{E}\left[1_{h(\mathbf{x'}) = \mathbf{y}} ~\left|~ \norm{\mathbf{x'}-\mathbf{x}}_p \le \epsilon\right.\right]\right].
    \end{equation}
    Since for any event $\mathbf{c}$, we have $P(\mathbf{c})=E[1_{\mathbf{c}}]$, we get our proof.
\end{proof}

Intuitively, \cref{thm:tower} explains the Tower robustness, i.e., we first randomly sample a neighbourhood (around an input-label pair) from the distribution, then we randomly pick a input-label pair from this neighbourhood, we ask, what will be the probability that we make a correct prediction on this picked pair? Next, we show the connection between the Tower robustness and $P\left(c_{\text{DC}}(\mathbf{x}, \mathbf{y}, p, \epsilon, h)\right)$.

\begin{corollary}
    DC accuracy is an unbiased estimator of a lower bound of Tower robustness. Formally, 
    \begin{equation}
        P\left(c_{\text{DC}}(\mathbf{x}, \mathbf{y}, p, \epsilon, h)\right) \le \operatorname{E}\left[P\left(h(\mathbf{x'}) = \mathbf{y} ~\left|~ \norm{\mathbf{x'}-\mathbf{x}}_p \le \epsilon\right.\right)\right]
    \end{equation}
\end{corollary}

\begin{proof}
    We have known that $\sum_i^n 1_{c_{DC, i}}/n$ is the sample mean of $P\left(c_{\text{DC}}(\mathbf{x}, \mathbf{y}, p, \epsilon, h)\right)$, such that our focus is to show the inequality above. It is sufficient if we can prove
    \begin{equation}
        \forall x, y, \lnot \left(1_{c_{\text{DC}}({x}, {y}, p, \epsilon, h)} > P\left(h(\mathbf{x'}) = {y} ~\left|~ \norm{\mathbf{x'}-{x}}_p \le \epsilon\right.\right)\right).
    \end{equation}
    Since $1_{(*)}$ returns only 0 or 1, and probabilities are always non-negative, we need only consider the case where the probability is strictly less than 1. In this case, there must be a positive measure of adversarial examples, which implies that $c_{\text{DC}}$ is false. This completes the proof.
\end{proof}

Intuitively, when trying to maximise the Tower robustness, we can still use DC accuracy as a sound lower bound to provide guarantee. Yet, neither probabilistic robustness $\operatorname{E}[c_{\text{PR}}(\mathbf{x}, \mathbf{y}, p, \epsilon, h, \kappa)]$ nor probabilistic robust accuracy $\operatorname{E}[c_{\text{PRA}}(\mathbf{x}, \mathbf{y}, p, \epsilon, h, \kappa, \alpha)]$ are naturally lower bounds of Tower robustness, i.e., they may overestimate.

\subsubsection{Tower robustness from a convolutional perspective}

We now draw a link from Tower robustness to commonly adopted correctness, i.e., expression~(\ref{eq:correct}). The link between the robustness of a given distribution and the correctness of a convolved distribution has been studied recently. Here, we formally extend this perspective to Tower robustness.

As we are sampling $\mathbf{x},\mathbf{y}$ (to get its neighbourhood) from the given distribution and then sample within the neighbourhood, it could be understood as adding two random variables. The first random variable is $\mathbf{x}$, representing the input from the joint distribution. We denote the second random variable as $\mathbf{t}\in\mathbb{X}$ and $\norm{\mathbf{t}}_p\le \epsilon$. Unsurprisingly, we can re-express $\mathbf{x'}\coloneqq \mathbf{x}+\mathbf{t}$. Theoretically, sampling of $\mathbf{x}$ and sampling of $\mathbf{t}$ are independent. Therefore, the probability density function (PDF) of $\mathbf{x'}$ can be expressed as
\begin{equation}
    p_{\mathbf{x'}}(x') = \int_{\mathbb{X}} p_{\mathbf{x}}(x) p_{\mathbf{t}}(t-x) dx.
\end{equation}
Considering the joint distribution $D$ of $\mathbf{x},\mathbf{y}$, and that $\mathbf{t}$ is independent of $\mathbf{y}$ as well, we can also express the conditional distributions in convolution as follows.
\begin{equation}
    p_{\mathbf{x'},\mathbf{y}}(x', y) = \int_{\mathbb{X}} p_{\mathbf{x},\mathbf{y}}(x, y) p_{\mathbf{t}}(t-x) dx
\end{equation}
That is, $(\mathbf{x'},\mathbf{y})\sim D * p_{\mathbf{t}}$. In this way, we can rewrite Tower robustness from expression~(\ref{eq:tr-exp}) to the following form.
\begin{equation}
\label{eq:tr-conv}
    \operatorname{E}_{\mathbf{x}, \mathbf{y}\sim D * p_{\mathbf{t}}}[h(\mathbf{x})=\mathbf{y}]
\end{equation}
Note that \emph{only} for expression~(\ref{eq:tr-conv}) we use $\mathbf{x}$ instead of $\mathbf{x'}$ to denote the perturbed input because we assume the distribution has been changed so we no longer need to add perturbations again. As seen expression~(\ref{eq:tr-conv}) is in a highly consistent form with expression~(\ref{eq:correct}). In other words, maximising robustness is indeed maximising some special kind of correctness, and this depends on how much the user/researcher would like to ``melt'' the original given distribution.

\section{Global Bounds (Lower and Upper) on Tower Robustness}
\label{sec:bounds}

Tower robustness can be estimated from finite samples without bias, but we still need to control the variance of the estimators. To establish a stricter formal guarantee in practice, we introduce a lower bound of Tower robustness and propose to find its unbiased estimator. In the following, we present our method in a bottom-up style. First, we describe how to test each individual input-label pair from a given distribution. Then, we show how we can combine the test results to provide guarantees.

\subsection{Binomial Hypothesis Test With Exact Solution}

In this step, we are given a pair of known $x, y$, and our task is to know whether the proportion of mispredictions in this $(p,\epsilon)$-neighbourhood is less than or equal to a real value $\kappa$ where $0<\kappa<1/2$. Formally, we need to check if the following condition holds.
\begin{equation}
\label{eq:pr-expand}
    P\left(h(\mathbf{x'})\neq y~\left|~\norm{\mathbf{x}-x}_p \le\epsilon\right.\right)\le\kappa
\end{equation}

Observe that the inequality above is exactly $c_\text{PR}(x, y; p, \epsilon, h, \kappa)$, the event formula for probabilistic robustness~\cite{robey2022probabilistically}. There are a few potential solutions to this task, where the shared strategy is to treat $h(\mathbf{x'})\neq y$ as a random variable with a Bernoulli distribution.

From there, \citet{robey2022probabilistically} use the maximum likelihood estimator from a random sample, which in the Bernoulli case is the sample mean (with sample size 100). This method is fast, but prone to overfitting and high variance, especially with small samples or when the true probability is near 0 or 1. When estimates approach $\kappa$, distinguishing their relation to $\kappa$ becomes statistically unreliable, necessitating more precise estimation to ensure significance.

\paragraph{Hypothesis testing}
We can use hypothesis testing to better estimate $c_{PR}$. To formulate a hypothesis testing problem, we first let $p_\intercal$ denote the true parameter p in our Bernoulli distribution. As such, we can state the null hypothesis
\begin{equation}
    H_0: p_\intercal > \kappa
\end{equation}
and alternative hypothesis
\begin{equation}
    H_1: p_\intercal \le \kappa.
\end{equation}
Then, we statistically determine whether there is enough evidence to reject the null hypothesis in favour of the alternative hypothesis.

For example, suppose we want to test the following one-sided hypothesis: $H_0: p_\intercal > 0.01$, and $H_1: p_\intercal \le 0.01$. Assume that we take a sample of 30 binary outcomes and observe 2 successes (i.e., two misprediction cases out of 30). Then, we can determine if this 2 out of 30 is significant.

While hypothesis testing helps preserve statistical significance, not all tests are well-suited for this problem. For example, the Agresti-Coull confidence interval used by \citet{zhang2023proa} approximates binomial test, and may not perform well when the sample size is very small or the true proportion $p$ is close to zero or one, leading to incorrect inferences~\cite{agresti1998approximate}.

\paragraph{Exact Binomial Test}

We compute the p-value as the cumulative probability as follows, under the binomial distribution with $p_\intercal=\kappa$.

\begin{equation}
\text{p-value} = P( \mathbf{k} > k\mid \mathbf{k}\sim \operatorname{Bin}(n, \kappa) = \sum_{i=0}^{k} \binom{n}{i} \kappa^i (1-\kappa)^{n-i}
\end{equation}
Here, $n$ represents the sample size, and $k$ denotes the number of successes. Then, if the obtained p-value exceeds a significance level $\alpha$, e.g., 0.05, we fail to reject $H_0$. If the obtained p-value is less than or equal to the significance level $\alpha$, we can reject $H_0$ and adopt $H_1$.

\paragraph{Compare with Agresti–Coull Approximation}

The difference between a binomial test (with exact solution) and its Agresti–Coull approximation lies in the way of calculating the cumulative density function of a binomial distribution. Particularly, the Agresti–Coull approximation estimates $p_\intercal$ using
\begin{equation}
    \tilde {p} =\frac {k + z_\alpha^2/2}{n + z_\alpha^2}
\end{equation}
where $k$ denotes the number of successes, and  $z_{\alpha }=\operatorname {\Phi ^{-1}} \!\!\left(\ 1-{\tfrac {\!\ \alpha \!\ }{2}}\ \right)$ is the quantile of a standard normal distribution, e.g., for $\alpha =0.05$ at one side, $z_{.05}=1.645$. A lower bound of the confidence interval is then calculated as
\begin{equation}
    \tilde{p} - z_\alpha\sqrt{\frac{\tilde{p}(1-\tilde{p})}{n + z_\alpha^2}}.
\end{equation}
If and only if this lower bound is greater than $\kappa$, then we can reject $H_0$ and adopt $H_1$.

In the example above, where $n=30, k = 2$, the exact binomial test gives a p-value of 0.9967, which fails to reject. In contrast, the Agresti–Coull confidence interval has a lower bound 
0.0153, which exceeds 0.01 and thus leads to rejection of $H_0$. This discrepancy illustrates that the Agresti–Coull approximation can overstate evidence against the null in small samples and when the $p_\intercal$ is near 0 or 1, while the exact binomial test remains conservative.

\paragraph{Both left- and right-tailed tests}

So far, we have been working on a left-tail hypothesis test, i.e., the null hypothesis is $p_\intercal$ being greater than some value. Intuitively, a left-tail test that rejects the null hypothesis suggests the probability of event $\mathbf{a_2}$ condition on event $\lnot\mathbf{a_1}$ is less than $\alpha$, where
\begin{itemize}
    \item Event $\mathbf{a_1}$: $c_\text{PR}(x, y; p, \epsilon, h, \kappa)$ is \emph{indeed} true,
    \item Event $\mathbf{a_2}$: Our evidence says that $c_\text{PR}(x, y; p, \epsilon, h, \kappa)$ would be true.
\end{itemize}
Intuitively, the significance level $\alpha$ also represents the failure rate of our conclusion.

Similar to a left-tail test, we can meanwhile plan a right-tail test. When a right-tail test rejects the null hypothesis, it suggests, the probability of event $\lnot\mathbf{a_2}$ conditioned on event $\mathbf{a_1}$ is less than $\alpha$.

For every input-label pair $(x,y)$, the hypothesis test will reject either side of the tail if there are sufficient samples. The estimation of sample number is well-studied~\cite{wald1992sequential,pocock1977group}, e.g., we could use the following formula
\begin{equation}
    n = \left( \frac{z_{1-\alpha} \cdot \sqrt{\kappa (1 - \kappa)} + z_{\alpha} \cdot \sqrt{\hat{p}_\intercal (1 - \hat{p}_\intercal)} }{\kappa - \hat{p}_\intercal} \right)^2
\end{equation}
Where $\hat{p}_\intercal$ is an empirically estimated Bernoulli parameter p. Therefore, we would obtain  
\begin{equation}
\label{ineq:fail-1}
    P(\lnot\mathbf{a_2}| \mathbf{a_1}) < \alpha,
\end{equation} and similarly:
\begin{equation}
\label{ineq:fail-2}
    P(\mathbf{a_2} | \lnot\mathbf{a_1}) <\alpha.
\end{equation}



\subsection{True Probability Rather Than Observed Events}

Next, we use the law of total probability, a special case of Tower Law, to eliminate the help values $\kappa$ and $\alpha$. We first eliminate $\alpha$.

\subsubsection{Eliminating $\alpha$}

Here, our goal is to determine the probability of $c_\text{PR}(\mathbf{x}, \mathbf{y}; p, \epsilon, h, \kappa)$ over the distribution. In other words, we want to obtain $P(\mathbf{a_1})$. Besides, what we already have is Inequality~(\ref{ineq:fail-1}) and (\ref{ineq:fail-2}).

\begin{lemma}
    We are given a joint distribution of $\mathbf{x}, \mathbf{y}$, where $\mathbf{x}$ has support $\mathbb{X}$ and $\mathbf{y}$ has support $\mathbb{Y}$. For any model $h:\mathbb{X}\to\mathbb{Y}$, its $(p,\epsilon)$-probabilistic robustness has the following lower and upper bounds.
    \begin{equation}
    \label{eq:bounds-alpha}
    \begin{aligned}
        P\left(c_\text{PR}(\mathbf{x}, \mathbf{y}; p, \epsilon, h, \kappa)\right) &> \frac{P\left(c_\text{PRA}(\mathbf{x}, \mathbf{y}; p, \epsilon, h, \kappa, \alpha)\right) - \alpha}{1+\alpha} \\
        P\left(c_\text{PR}(\mathbf{x}, \mathbf{y}; p, \epsilon, h, \kappa)\right) &< \frac{P(c_\text{PRA}(\mathbf{x}, \mathbf{y}; p, \epsilon, h, \kappa, \alpha))}{1-\alpha}
    \end{aligned}
    \end{equation}
\end{lemma}

\begin{proof}
    From the law of total probability, it is obtained that 
    \begin{equation}
        P(\mathbf{a_2}) = P(\mathbf{a_2}\mid\mathbf{a_1})P(\mathbf{a_1}) + P(\mathbf{a_2}\mid\lnot \mathbf{a_1})P(\lnot \mathbf{a_1}).
    \end{equation}
    
    If we further write $P(\lnot \mathbf{a_1}) = 1 - P(\mathbf{a_1})$, and $P(\mathbf{a_2}|\lnot \mathbf{a_1}) = 1 - P(\lnot \mathbf{a_2}|\lnot \mathbf{a_1})$, we eventually get
    \begin{equation}
        P(\mathbf{a_1}) = \frac{P(\mathbf{a_2}) - P(\mathbf{a_2}|\lnot \mathbf{a_1})}{1 - P(\lnot \mathbf{a_2}|\mathbf{a_1}) + P(\mathbf{a_2}| \lnot \mathbf{a_1})}.
    \end{equation}
        Now that we know that $0 < P(\mathbf{a_2}|\lnot \mathbf{a_1}), P(\lnot \mathbf{a_2}|\mathbf{a_1}) < \alpha$, we can find the lower and upper limit of $P(\mathbf{a_1})$ as 
    \begin{equation}
            \frac{P(\mathbf{a_2}) - \alpha}{1+\alpha} < P(\mathbf{a_1}) < \frac{P(\mathbf{a_2})}{1-\alpha}.
    \end{equation}
    Recall that the event $\mathbf{a_1}$ is essentially that $c_\text{PR}(\mathbf{x}, \mathbf{y}; p, \epsilon, h, \kappa)$ holds. Also, event $\mathbf{a_2}$ is the observation. This completes the proof.
\end{proof}

Inequality~(\ref{eq:bounds-alpha}) makes sense because the probabilistic robustness (although itself is not measurable) is still predominantly positively related to the observed probabilistic robust accuracy, and a smaller false positive rate ($\alpha$) also makes these two quantities closer. Note that the usually measurable probabilistic robust accuracy is a maximum likelihood estimator of $P(c_\text{PRA}(\mathbf{x}, \mathbf{y}; p, \epsilon, h, \kappa, \alpha))$, as it is computing the sample mean given discrete $\{(x_1, y1), \ldots\}$.

\subsubsection{Eliminating $\kappa$}
Although $\alpha$ has been eliminated, $\kappa$ remains. These extra parameters add confusion and labour to robustness evaluation, e.g., fair comparison is difficult if different $\kappa$ values are used. Next, we present how to eliminate $\kappa$.

In this task, we are given $P(c_\text{PR}(\mathbf{x}, \mathbf{y}; p, \epsilon, h, \kappa))$ or its lower and upper bounds. We can equivalently expand this expression back to the following form using expression~(\ref{eq:pr-expand}).
\begin{equation}
    P\left(P\left(h(\mathbf{x'})\neq y~\left|~\norm{\mathbf{x}-x}_p \le\epsilon\right.\right)\le\kappa\right)
\end{equation}
In the following, we link this form back to Tower robustness, through both Markov's Inequality and Tower Law.

\begin{theorem}
    We are given a joint distribution of $\mathbf{x}, \mathbf{y}$, where $\mathbf{x}$ has support $\mathbb{X}$ and $\mathbf{y}$ has support $\mathbb{Y}$. For any model $h:\mathbb{X}\to\mathbb{Y}$, its $(p,\epsilon)$-Tower robustness has a computable upper bound as
    \begin{equation}
    \label{eq:upper}
        \frac{\kappa P(c_\text{PRA}(\mathbf{x}, \mathbf{y}; p, \epsilon, h, \kappa, \alpha))}{1-\alpha}-\kappa + 1.
    \end{equation}
\end{theorem}
\begin{proof}
    According to Markov's Inequality, for a non-negative function $f(b)\ge 0$ and a random variable $\mathbf{b}$,
    \begin{equation}
        P(f(\mathbf{b})\ge \kappa) \le \frac{\operatorname{E}[f(\mathbf{b})]}{\kappa}.
    \end{equation}
    This gives us
    \begin{equation}
        P(f(\mathbf{b})< \kappa) \ge 1-\frac{\operatorname{E}[f(\mathbf{b})]}{\kappa}.
    \end{equation}
    Observe that $P(h(\mathbf{x'})\neq y\mid\norm{\mathbf{x}-x}_p \le\epsilon)$ is a non-negative function of $\mathbf{x}$. Thus, substituting it in we get the following form.
    \begin{equation}
        P\left(P_{\norm{\mathbf{x}-x}_p \le\epsilon}\left(h(\mathbf{x'})\neq y\right)\le \kappa\right)\ge 1 - \frac{\operatorname{E}[P_{\norm{\mathbf{x}-x}_p \le\epsilon}\left(h(\mathbf{x'})\neq y\right)]}{\kappa}
    \end{equation}
    Shifting the sides, we get
    \begin{equation}
    \begin{aligned}
        \operatorname{E}&\left[P_{\norm{\mathbf{x}-x}_p \le\epsilon}\left(h(\mathbf{x'})\neq y\right)\right]\\
        &\ge \kappa\left(1 - P\left(P_{\norm{\mathbf{x}-x}_p \le\epsilon}\left(h(\mathbf{x'})\neq y\right)\le \kappa\right)\right)
    \end{aligned}
    \end{equation}
    Multiplying -1 and then adding 1 to both sides, we get
    \begin{equation}
    \begin{aligned}
        \operatorname{E}&\left[P_{\norm{\mathbf{x}-x}_p \le\epsilon}\left(h(\mathbf{x'})= y\right)\right]\\
        &\le 1- \kappa\left(1 - P\left(P_{\norm{\mathbf{x}-x}_p \le\epsilon}\left(h(\mathbf{x'})\neq y\right)\le \kappa\right)\right)\\
        &\le 1- \kappa + \kappa P\left(P_{\norm{\mathbf{x}-x}_p \le\epsilon}\left(h(\mathbf{x'})\neq y\right)\le \kappa\right)
    \end{aligned}
    \end{equation}
    We have got the lower and upper limits of the third term on the right-hand side of this inequality. Substituting that in, we get
    \begin{equation}
        \operatorname{E}\left[P_{\norm{\mathbf{x}-x}_p \le\epsilon}\left(h(\mathbf{x'})= y\right)\right] < \frac{\kappa P(c_\text{PRA}(\mathbf{x}, \mathbf{y}; p, \epsilon, h, \kappa, \alpha))}{1-\alpha}-\kappa + 1.
    \end{equation}
    This provides an upper bound on Tower robustness of the given model over the given distribution.
\end{proof}

\begin{theorem}
    We are given a joint distribution of $\mathbf{x}, \mathbf{y}$, where $\mathbf{x}$ has support $\mathbb{X}$ and $\mathbf{y}$ has support $\mathbb{Y}$. For any model $h:\mathbb{X}\to\mathbb{Y}$, its $(p,\epsilon)$-Tower robustness has a computable lower bound as
    \begin{equation}
        (1-\kappa) \frac{P\left(c_\text{PRA}(\mathbf{x}, \mathbf{y}; p, \epsilon, h, \kappa, \alpha)\right) - \alpha}{1+\alpha}.
    \end{equation}
\end{theorem}

\begin{proof}
    We start from the Tower Law, and we can get the following inequality
    \begin{equation}
    \begin{aligned}
        \operatorname{E}&\left[P_{\norm{\mathbf{x}-x}_p \le\epsilon}\left(h(\mathbf{x'})\neq y\right)\right] \le\\
        &\kappa P\left(P_{\norm{\mathbf{x}-x}_p \le\epsilon}\left(h(\mathbf{x'})\neq y\right)\le \kappa\right) +\\
        &\kappa^* P\left(P_{\norm{\mathbf{x}-x}_p \le\epsilon}\left(h(\mathbf{x'})\neq y\right)> \kappa\right)
    \end{aligned}
    \end{equation}
    where $\kappa^*$ is a known upper bound that $\kappa^* \ge P_{\norm{\mathbf{x}-x}_p \le\epsilon}\left(h(\mathbf{x'})\neq y\right)$ on the function's output. A naive $\kappa^*$ value is 1, which leads to 
    \begin{equation}
    \begin{aligned}
        \operatorname{E}&\left[P_{\norm{\mathbf{x}-x}_p \le\epsilon}\left(h(\mathbf{x'})\neq y\right)\right] \le\\
        &(\kappa-1) P\left(P_{\norm{\mathbf{x}-x}_p \le\epsilon}\left(h(\mathbf{x'})\neq y\right)\le \kappa\right) + 1
    \end{aligned}
    \end{equation}
    Multiplying -1 and then adding 1 to both sides, we get
    \begin{equation}
    \begin{aligned}
        \operatorname{E}&\left[P_{\norm{\mathbf{x}-x}_p \le\epsilon}\left(h(\mathbf{x'})= y\right)\right]\\
        &\ge (1-\kappa) P\left(P_{\norm{\mathbf{x}-x}_p \le\epsilon}\left(h(\mathbf{x'})\neq y\right)\le \kappa\right)
    \end{aligned}
    \end{equation}
    Substituting in a lower bound of the right-hand side, we get
    \begin{equation}
    \begin{aligned}
        \operatorname{E}&\left[P_{\norm{\mathbf{x}-x}_p \le\epsilon}\left(h(\mathbf{x'})= y\right)\right]\\
        &> (1-\kappa) \frac{P\left(c_\text{PRA}(\mathbf{x}, \mathbf{y}; p, \epsilon, h, \kappa, \alpha)\right) - \alpha}{1+\alpha}.
    \end{aligned}
    \end{equation}
    This completes the proof.
\end{proof}

Till here, we have successfully eliminated the helper parameters $\kappa$ and $\alpha$. In practice, we can freely choose convenient values of them and compare the model robustness of, say, $h_1$ with another model $h_2$ that is evaluated with different helper parameters. Essentially, we are comparing the lower bound, or the robustness guarantee of each model. And, when one intends to claim a better model robustness, it is reasonable to push the lower bound tighter by using smaller $\alpha$ and $\kappa$. This is reasonable, compared to the existing practice that choosing higher $\kappa$ and $\alpha$ inflates the robustness score~\cite{zhang2023proa}.

Another advantage of using Tower robustness, or its lower bound guarantee, to evaluate robustness, is that in this unified framework, we can compare the deterministic robustness and probabilistic robustness, i.e., we find a fair comparison method. Furthermore, the proposed lower bound can be used together with deterministic robustness verification, i.e., for those cases that are not able to be deterministically verified, we can compute their probabilistic robust accuracy and combine them. Such a combination is not merely a score adding, but is derivable from the governing expression of Tower robustness.

\section{Experiment}
\input{figt}

\paragraph{Setup}
We compute the Tower robustness upper and lower bounds for existing neural networks. There are altogether six neural models, trained from ERM~\cite{vapnik1999nature}, PGD adversarial training~\cite{madry2017towards}, TRADES~\cite{zhang2019theoretically}, MART~\cite{wang2019improving}, CVaR~\cite{robey2022probabilistically}, and small-box certified training~\cite{muller2022certified}.

We evaluate these models on the commonly used MNIST~\cite{lecunmnist} and CIFAR-10 tasks. Typically, MNIST classification is tested within $(\infty,1/10)$-neighbourhood and CIFAR-10 $(\infty,8/255)$-neighbourhood, which is also our default setting. For each dataset, we evaluate on the 10,000 images in the testing set.

\paragraph{Result}
We present our experimental results in \cref{fig:tr}. Since the validity of Tower robustness is formally shown in previous discussions, our experiments are primarily used to help understand and provide visual ideas. From \cref{fig:tr}, we may first observe that CVaR so far obtains the highest lower bound of Tower robustness on CIFAR-10, over 90.6\%. The second highest is from PGD adversarial training, but only slightly above 80\%. On MNIST, except for ERM, which does not optimise robustness and Small-box, which has concentrated primarily on deterministic robustness, all other models achieve a pretty high guarantee. These results aligned with our intuition and others' results~\cite{li2023sok}. Again, in a strict sense, this alignment would not lead us to any implication, but simply we are more or less not off track.

What brings us more benefits is to observe the trend when the helper parameters are decreasing. As mentioned in previous discussions, if a model provider would like to get a better certificate of the model's robustness, a fair approach is to squeeze these helper parameters. This lifts the lower bound, as observed in \cref{fig:tr}, but does not invalidate it. Intuitively, this could be considered as different verification methods in the deterministic robustness context, where each model does not necessarily obtain a certificate in the same approach. In our experiment, we use a default $\kappa = 1/10, \alpha = 1/10$. When reducing either $\kappa$ or $\alpha$, the lower bound can grow, which overcomes the inflation problem prevalent in probabilistic robustness evaluations. Note that reducing $\kappa$ turns out to be more effective than reducing $\alpha$. Intuitively, $\kappa$ is more closely related to the proportion of adversarial examples than $\alpha$ is.

Intriguingly, we can apply the probabilistic robustness verification techniques to the test cases not deterministically verified. For example, Small-Box is a (deterministic) certified training method, which is usually evaluated only for deterministic robustness. While one could indeed subject this model to probabilistic robustness verification, there have been no theoretical guarantees provided for combining both. Our method is the first that put these two verification results together and provide an interpretable quantity. Yet, because the Small-Box model has primarily been tuned towards deterministic verification, its Tower robustness lower bound remains one of the lowest on the list, even if we simultaneously enable probabilistic robustness verification.

As for the upper bound of Tower robustness, it generally would not be used as a guarantee. But it can be used to confirm the defects of a given model. Its monotonicity is not obvious in expression~(\ref{eq:upper}), but we have observed it likely decreases when either $\kappa$ or $\alpha$ grows.

\section{Related Work}

\paragraph{Empirical-based Approaches}
A common and intuitive method for assessing model robustness is the test error rate~\cite{carlini2017adversarial}, valued for its speed and versatility in evaluating both $L^p$ norm and transformation perturbations (e.g., rotation, illumination). Adversarial attacks exemplify this approach, identifying vulnerable samples within defined vicinities~\cite{carlini2017adversarial,kurakin2016adversarial}. However, this method lacks guarantees, as previously validated robustness may be undermined by stronger attacks, necessitating re-evaluation across models. To mitigate this,~\cite{hendrycks2021natural} proposes using a standardized adversarial test set, enabling efficient and consistent evaluation across models. Although this technique doesn't strictly align with any formal robustness definition and risks overfitting, it offers a practical, rapid method for accurately ranking models.

\paragraph{Reachability-based algorithms} take an input and a predefined perturbation to assess whether a given model meets specific safety criteria, such as robustness and boundary adherence. This evaluation of robustness can be achieved by addressing output range analysis problems through computational methods including interval analysis, linear programming, and optimization~\cite{virmaux2018lipschitz,lee2020lipschitz,li2019analyzing,yang2021reachability}. Layerwise analysis is a prevalent approach in this context. The K-ReLU technique~\cite{singh2019beyond}, which utilizes joint relaxation, effectively captures dependencies between inputs to different ReLUs within a layer, thus overcoming the convex barrier imposed by single neuron triangle relaxation and its approximations. Another notable method, CLEVER~\cite{weng2018evaluating}, transforms the robustness analysis into a local Lipschitz constant estimation problem and employs the Extreme Value Theory for efficient evaluation. However, it is important to note that reachability analysis techniques often require Lipschitz continuity over outputs for the target networks, which may limit their applicability in some instances.

\paragraph{Constraint-based Verification}  aims to ensure that a given model adheres to a predefined set of robustness constraints. This process entails defining the constraints and then analyzing the model using techniques such as model checking, testing, or runtime monitoring~\cite{katz2017reluplex,amir2021smt}. Marabou, for instance, translates queries into constraint satisfaction problems using Satisfiability Modulo Theories (SMT)\cite{barrett2018satisfiability}. It is compatible with various network activation functions and topologies, and it carries out high-level reasoning on the network~\cite{katz2019marabou}. Another SMT-based approach, Reluplex, employs the simplex method while extending it to accommodate the Rectified Linear Unit (ReLU) activation function~\cite{katz2017reluplex}. Although constraint-based verification can offer formal guarantees that a model fulfils certain constraints or requirements, challenges such as computational complexity and pessimistic results have hindered its widespread application~\cite{zhang2023proa}. Consequently, researchers have proposed probabilistic robustness to address these limitations.

\paragraph{Probabilistic-Based Approach} Diverging from the deterministic verification methods previously examined, probabilistic approaches consider the balance between performance and sample complexity in both worst-case and average-case learning situations~\cite{lecuyer2019certified,cohen2019certified,robey2022probabilistically}. \citet{cohen2019certified} employ randomized smoothing as a reachability method, calculating a lower bound on robustness by adding noise to input data and observing output behaviour. This lower bound can then be determined by the proportion of perturbed inputs yielding robust outputs. Additionally, \citet{salman2019provably} develops an adapted attack for smoothed classifiers and utilises adversarial training to enhance the provable robustness of these classifiers. However, these statistical methods primarily concentrate on pixel-level additive perturbations and often require assumptions about target neural networks or input distributions, limiting their broad applicability as they mostly function with $L^2$ norms. Moreover, \citet{zhang2023proa} introduces PRoA, which relies on adaptive concentration. Although they employ an approximation method for distribution, it may lead to soundness issues in the future. In contrast, the proposed exact binomial test method retains the benefits of probabilistic approaches while using an exact solution to circumvent soundness problems. Similar to PRoA, our method can address black-box deep learning models without assumptions and is not restricted by transformation attacks (unbounded by $L^p$ norms).

Therefore, there is a growing need for certifiably robust approaches, including both certified robustness verification~\cite{li2023sok} and certified training~\cite{muller2022certified,shi2021fast}.

\section{Conclusion}

In this study, we introduce an innovative approach aimed at enhancing the observation process of probabilistic robustness in neural networks against adversarial examples. Our method comprises two key components: a unified framework that defines robustness in a global perspective that is built on tower law, and a computable lower bound to guarantee adversarial robustness. Theoretical analysis reveals that existing robustness metrics often fundamentally disagree with each other, i.e., each metric captures part of the robustness property and sticks to that. However, it is crucial to obtain an interpretable robustness measurement that fairly compares different models or systems.

Throughout this work, we adopt a global perspective to unify the deterministic and probabilistic robustness, i.e., we focus on the behaviour of an upcoming unseen sample. This property can be naturally lower bounded by the deterministic robust accuracy, resolving the incompatibility between a strict guarantee and practical implication. We also show the convolutional interpretation of Tower robustness, highlighting its consistency with correctness.

Moreover, we provide computable lower and upper bounds of the proposed Tower robustness. By eliminating the excessive degree of freedom in probabilistic robustness, we solve the overestimation problem of probabilistic robustness metrics, resulting in more intuitive and reliable outcomes when certifying robustness on the test set. Furthermore, we can combine the proposed bounds with deterministic robustness, and this still provides theoretical guarantees.

In summary, our approach presents a significant step forward in improving the formal evaluation of the general sense of robustness in neural networks against adversarial examples. The combination of exact calculations, degree-of-freedom reduction, and standardised thresholds contributes to a more reliable and intuitive certification process.

In future, we will also explore the use of differential privacy for safeguarding against adversarial attacks~\cite{Mu-CIKM25b}.

{
\vspace{3mm}
\footnotesize
\noindent{\textbf{Acknowledgements}}.
This research is supported in part by the Ministry of Education, Singapore, under its Academic Research Fund Tier 2 (Award No. MOE-T2EP20123-0015). Any opinions, findings and conclusions, or recommendations expressed in this material are those of the authors and do not reflect the views of the Ministry of Education, Singapore.

\section*{GenAI Usage Disclosure}
Generative AI tools were used for minor language editing, such as grammar/spelling checks and minor changes, on original author-written text.

\bibliography{favourite,anthology,custom}
\end{document}

%% file: customised-commands.tex
\usepackage[english]{babel}

\usepackage{graphicx}
\usepackage{microtype}
\usepackage{booktabs}

\usepackage{float}
\usepackage{wrapfig}
\usepackage[normalem]{ulem}

\usepackage{bm}
\usepackage{textcomp}
\usepackage{amsmath}
\usepackage{amsthm}
\usepackage{amsfonts}
\usepackage{mathtools}
\usepackage{pifont}

\usepackage{algorithm}
\usepackage{algorithmic}
\usepackage{gensymb}

\usepackage{textcomp}
\usepackage{multirow}
\usepackage{lscape}
\usepackage{subcaption}

\usepackage{mdframed}
\usepackage{hyperref}
\usepackage{textgreek}
\usepackage{tabularx, colortbl, xcolor}
\definecolor{mygray}{gray}{0.95}
\definecolor{mycyan}{HTML}{005397}
\definecolor{myred}{HTML}{E13333}
\definecolor{mymagenta}{HTML}{BF3E87}
\definecolor{mypurple}{HTML}{1B2278}

\definecolor{tearose}{HTML}{F584C5}
\definecolor{coral}{HTML}{F67088}
\definecolor{dodger_blue}{HTML}{3BA3EC}
\definecolor{domino}{HTML}{BC9F48}
\definecolor{domino}{HTML}{BC9F48}
\definecolor{domino}{HTML}{BC9F48}
\definecolor{catalina_blue}{HTML}{1C3168}
\definecolor{catalina_blue}{HTML}{1C3168}
\definecolor{catalina_blue}{HTML}{1C3168}
\definecolor{dark_scarlet}{HTML}{C63D52}
\definecolor{cerulean}{HTML}{0192A8}

\definecolor{tussock}{HTML}{C99E31}
\definecolor{p13}{HTML}{BFB5D7}
\definecolor{b14}{HTML}{BEA1A5}
\definecolor{y15}{HTML}{F0Cf61}
\definecolor{Merino}{HTML}{F3EEE3}

\newcolumntype{a}{>{\columncolor{p13}}l}

\usepackage[capitalize,noabbrev]{cleveref}
\crefname{ineq}{Inequality}{Inequalities}
\creflabelformat{ineq}{#2{\upshape(#1)}#3}

\theoremstyle{remark}

\newtheoremstyle{researchquestionstyle}{}{}{}{}{\bf}{}{.5em}{}
\theoremstyle{researchquestionstyle}

\usepackage{braket}
\newcommand{\norm}[1]{\left\lVert#1\right\rVert}

\usepackage{stmaryrd}

\DeclareMathOperator*{\argmax}{arg\,max}

\usepackage{tikz}
\usetikzlibrary{trees}
\usepackage{tikz-dependency}
\usepackage{tikz-3dplot}
\usetikzlibrary{chains,scopes}
\usetikzlibrary{arrows.meta,automata,positioning}
\usetikzlibrary{shapes.geometric, arrows}
\usepackage{pgfplots}

\pgfplotsset{
  every axis/.append style = {thick},
  tick style = {thick,black},
  /tikz/normal shift/.code 2 args = {%
    \pgftransformshift{%
        \pgfpointscale{#2}{\pgfplotspointouternormalvectorofticklabelaxis{#1}}%
    }%
  },%
  shift/.style = {
    tick align        = outside,
    scaled ticks      = false,
    enlargelimits     = false,
    ticklabel shift   = {#1},
    axis lines*       = left,
    xtick style       = {normal shift={x}{#1}},
    ytick style       = {normal shift={y}{#1}},
    x axis line style = {normal shift={x}{#1}},
    y axis line style = {normal shift={y}{#1}},
  },
  shift/.default = 10pt,
  shift3d/.style = {
    shift=#1,
    ztick style       = {normal shift={z}{#1}},
    z axis line style = {normal shift={z}{#1}},
  },
  shift3d/.default = 10pt,
}

\usepgfplotslibrary{fillbetween}

\usepackage{listings}
\usepackage{array}
\newcolumntype{H}{>{\setbox0=\hbox\bgroup}c<{\egroup}@{}}

\usepackage[most]{tcolorbox}
\newtcolorbox[auto counter]{summary}[1][]{title={\bfseries Summary~\thetcbcounter},enhanced,
  coltitle=black,
  colback=white,
  top=0.3in,
  attach boxed title to top left=
  {xshift=1.5em,yshift=-\tcboxedtitleheight/2},boxrule=0.5pt,  sharp corners, fonttitle=\bfseries,boxed title style={size=small,colback=white,colframe=white},#1}

%% file: figt.tex
\begin{figure*}[t]
\centering
\scriptsize

\includegraphics[width=\linewidth]{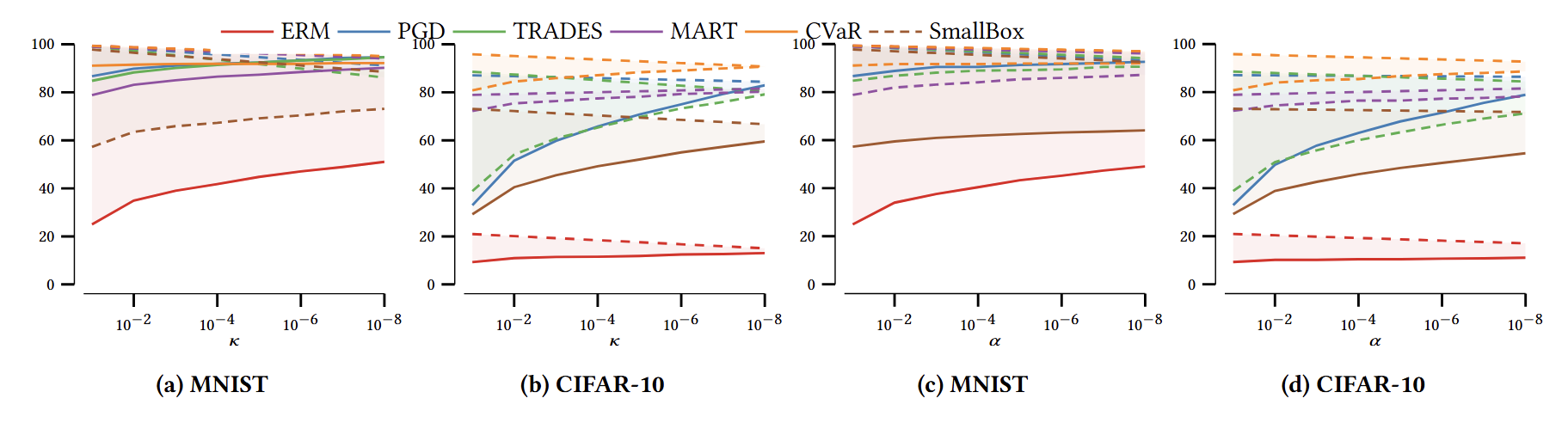}

\caption{Tower robustness guarantees. Each solid line shows the lower bound, and the dashed line shows the upper bound.}
\label{fig:tr}
\end{figure*}